\documentclass[journal]{IEEEtran}

\usepackage[style=ieee]{biblatex} 
\bibliography{mybibliography.bib}
\usepackage{hyperref}
\usepackage{amsmath}
 \usepackage{multirow}
 \usepackage{siunitx}
 \usepackage{soul}
 \usepackage{graphicx} 
\usepackage{amsmath,amssymb,amsthm,epsfig}
\newtheorem{theorem}{Theorem}[section]
\usepackage{hyperref}
\usepackage[style=ieee]{biblatex} 
\usepackage{xcolor}
\usepackage[font=small,labelfont=bf]{caption}

\usepackage{url}
\usepackage{graphicx}  
\usepackage{float}  
\usepackage{booktabs} 
\usepackage{array} 
\newcolumntype{R}{>{\raggedleft\arraybackslash}p{1.5em}} 
\usepackage{graphicx} 
\usepackage{bm}

\begin{document}
\title{{On Using the Shapley Value for Anomaly Localization:  A Statistical Investigation}\thanks{ This work was supported by the U.S. Office of Naval Research under Grant N00014-22-1-2626.}}  
\author{Rick S. Blum, Xubin Fang, and Franziska Freytag \thanks{Rick S. Blum, Xubin Fang, and Franziska Freytag are with the Electrical and Computer Engineering  Department of Lehigh University (emails: frf223@lehigh.edu, rblum@eecs.lehigh.edu). }}
\maketitle

\begin{abstract}
Recent publications have suggested using the Shapley value for anomaly localization for sensor data systems.  { We use a reasonable statistical model for 
the classifiers required to compute the Shapley value to provide  
repeatable and rigorous analysis
in the anomaly localization application. Then we provide a proof that using a single fixed term in the Shapley value calculation  achieves a lower complexity anomaly localization test, with the same probability of error,  as a test using the Shapley value in  cases with independent observation.  While it is impossible to test all possible cases numerically, we found this to be true in all the  cases we tested with independent observations.    For some dependent observation cases with two sensors,  where only the second sensor data is anomalous, we show numerically that the Shapley value test can falsely decide an anomaly occurs at the first (nonanomalous) sensor with a probability which approaches one for increasing anomaly magnitude. On the other hand, using a single fixed term in the Shapley value calculation in these cases gives a reasonably small probability of an anomaly occurring at the first (nonanomalous) sensor for any anomaly magnitude.
These results are the first of this type we have seen, could encourage new algorithm development, and should encourage future research to more fully understand these observations.  A better understanding of the Shapley value, given its popularity, seems an important topic which could lead to improvements 
in algorithms and real implementations in the future. }
\end{abstract}


\section{Introduction}
The incorporation of sensors into many systems provides important advantages \cite{Varshney,Willett,BChen,RNiu,RVish,LKaplan,Visa}. 
Sensor data is highly vulnerable to cyber attacks and cyber attacks on sensor data can cause tremendous damage.  
Unfortunately, protection against such cyber attacks on sensor data has not been adequately addressed \cite{sensors-security}. 
This problem becomes even more important given the emergence of the internet of things, which incorporates sensors to an even greater extent \cite{GRR2020}. 

Some recent papers   \cite{E-sfd,ameli2022unsupervised}  described the very interesting idea of using the Shapley value, a quantity that has received considerable attention in the game theory and machine learning communities \cite{li2023surveyexplainableanomalydetection}, in a new way that could be very useful for sensor system security. 
The idea in 
\cite{E-sfd,ameli2022unsupervised} 
is to use the Shapley value to determine if the data at a particular sensor is anomalous, thus localizing the anomaly (or cyber attack).  
We further investigate this topic here, in a controlled setting, to better understand some basic related issues.

Assume we have $N$ sensors, each providing an observation, and we denote the whole set of observations by $ x_1, x_2, ..., x_N$.  
If we want to calculate the Shapley value for the observation $x_i, 1 \leq i \leq N$, the calculation is (see explanation in 
\cite{watson2023explainingpredictiveuncertaintyinformation})   
\begin{equation}\label{shap}
    \phi (x_i)= \sum_{S \subseteq {\cal N} / (i)} \frac{ |S| ! (N- |S| -1)!}{N!} (v (S \cup (i)) - v(S))
\end{equation} 
where:
{ 
\begin{itemize}
    \item $\phi(x_i)$ is the Shapley value for the $i$th sensor observation $x_i$ \item ${\cal N} = \{ 1, 2, ..., N\}$ is the set of all possible sensor indices
    \item $v$ is a function representing a soft classifier usually derived 
    from machine learning 
    \item $S$ denotes a subset of the sensor indices in  ${\cal N}$. 
    \end{itemize} 
We note that (\ref{shap}) uses the standard  notation of $v$ used in the Shapley value definition  \cite{watson2023explainingpredictiveuncertaintyinformation}. However, each instance of $v$ in (\ref{shap}) generally employs a different number of arguments and sensor inputs to $v$. Thus $v$ generally stands for a different function in each instance. 
Being a soft classifier, 
the output value of $v$  indicates the likelihood that an anomaly is present in the set of sensors which have indices in the set which is the argument to $v$. 
{ In particular, $v$ should produce a large positive output if any of the sensor data input to it are anomalous.  Alternatively, 
$v$ should produce a large negative output if none of the sensor data input to it are anomalous.  }

}

Note that in (\ref{shap}), the sum is over all possible subsets $S $ of  sensors with indices chosen from ${\cal N}$ which exclude sensor $i$. Each term in the sum in (\ref{shap}) 
involves two quantities. 
The first quantity, $\frac{ |S| ! (n- |S| -1)!}{n!}$, is a weighting factor which depends on the cardinally of the set $S$, denoted as $|S|$,  where $S$ corresponds to the value employed in the corresponding term in the sum in (\ref{shap}). 
{ The quantity $v (S \cup (i)) - v(S)$ in (\ref{shap}) involves the subtraction of two terms dependent on  the subset $S$. However the two arguments to $v$ differ by the element $i$ so this quantity measures the impact of seeing the value of sensor element $i$. 

In this paper we focus on cases where the anomalies are due to attacks on the sensor data.
Due to the difficulty in obtaining training data describing all possible attacks on all possible subsets of sensor data, we focus on anomaly/attack 
localization which is deployed based only on unattacked training data, which is common.  No anomalous/attacked training data is available.  { To allow repeatable and rigorous analysis, we employ a mathematical  model for the classifier $v$ that has often been used in the past as a statistically described anomaly detector. 
We employ the model for $v() $  as the natural log of the reciprocal of the unattacked joint probability density function (pdf) of the sensor data corresponding to the indices in the input to $v$ if the sensor data 
are all modeled as continuous random variables. 
If the sensor data are all modeled as discrete random variables,   we model $v$  as the natural log of the reciprocal of the unattacked joint probability mass function (pmf) of the sensor data corresponding to the indices in the input to $v$.  Note that these joint pdfs or pmfs can be learned from the assumed training data, but we assume them known in our analysis and tests to promote repeatable analysis. } It should also be noted that this approach allows an analytical formulation (thus highly controllable) for $v$ for any subset of sensor data and this formulation makes sense intuitively as we explain next.  

{  

Such a 
$v$  function will produce a more negative value (signifying no attack) when its sensor 
arguments occur with high  probability under the 
unattacked joint pdf/pmf of  
these sensor arguments, which signifies these sensor arguments 
are more  likely an unattacked data sample. 
When these sensor arguments 
occur with lower  probability 
under the unattacked joint pdf/pmf, $v$ gives a more positive  value, signifying a higher  probability of an attack.

This leads to an interesting question, which we focus on in this { work}.  We ask if it better to employ the Shapley value for anomaly localization or to employ $v(i)$, which we now denote as $v(x_i)
\overset{\Delta}{=} v(i)$  to make it clear that $v(i)$ represents the classifier with input $x_i$. }
To answer this question, we compare the probability of error of two tests which each make a decision on if the anomaly includes  the $i$th sensor.  Each test will decide that the anomaly includes sensor $i$ if the function it compares to a threshold is larger than the threshold. Otherwise the test decides the anomaly does not include sensor $i$. The first test compares $ \phi(x_i)$ to an optimized threshold chosen to minimize the probability of error of this test.  The second test compares $ v(x_i)$ to an optimized threshold chosen to minimize the probability of error of this test.  
For numerical results, since we generate  attacks, we will know if the anomaly includes each sensor.  Based on standard statistical theory, if the test employing $ \phi(x_i)$ (or $ v(x_i)$) gives smaller probability of error, then $ \phi(x_i)$ (or $ v(x_i)$) 
is better for anomaly localization.  

{ 
Now, we address the complexity\footnote{The time complexity to be exact, which measures the number of operations an algorithm performs relative to the input size ($N$).} of computing $ v(x_i)$ and $ \phi(x_i)$ in practice. 
To obtain $ v(x_i)$, 
we just need to learn a single function and evaluate it once.  To compute $ \phi(x_i)$, we 
need to learn roughly $ \sum_{j=0}^{N} \binom{N}{j} $ functions, as per (\ref{shap}), and perform 
$O(2^N)$ computations given evaluations of all those functions.  Even ignoring the larger number of functions you must learn, it is always more complex to compute $ \phi(x_i)$ and the increase in complexity grows with $N$.}


Surprisingly, our numerical results show that comparing $ v(x_i)$ to an optimized threshold performs equivalent in terms of probability of error (to $ \phi(x_i)$) 
when we use the described formulation for all the cases with independent observations we have considered. We give an analytical proof showing this must be true for all independent observation cases. 
Thus, for independent observation cases, using $ v(x_i)$ performs as well as using $ \phi(x_i)$, with lower complexity for the reasonable formulation considered. 
{ 
For some dependent observation cases with two sensors,  where only the second sensor data is anomalous, we show numerically that the Shapley value test can falsely decide an anomaly occurs at the first (nonanomalous) sensor with a probability which approaches one for increasing anomaly magnitude. On the other hand, using $ v(x_i) $ always gives reasonably small probability of an anomaly occurring at the first (nonanomalous) sensor in these cases. }

\section{Literature Review}
{ 
{ \subsection{Shapley Value: Game Theory, Explainable AI, and Anomaly Detection}}
\par The Shapley value 
\cite{Shapley1953} stems from game theory where the formula for a singular Shapley value per player of a game indicates a coalition between the multiple players, distributing total gain.
The more players or members a game has, the more complex and time consuming the calculation becomes, making it very challenging for large systems. Recently, the Shapley value has been used in machine learning in order to explain results from algorithms, as can be seen in many references in \cite{DeepLearningADAreview, DeppLearningADAsurvey}. 
In \cite{li2023surveyexplainableanomalydetection}, an overview is provided of how the Shapley value  and other alternative methods are used in  explainable anomaly detection.  On the other hand, there are many papers related to anomaly detection that do not specifically consider the Shapley value, see the references in  \cite{DeepLearningADAreview, DeppLearningADAsurvey} for example.

{\subsection{Shapley Value in Sensor  Anomaly Localization}}
We previously mentioned that \cite{E-sfd,ameli2022unsupervised} suggested using the Shapley value in sensor anomaly localization. 
In \cite{E-sfd}, the authors employ a simplified version of the Shapley value to pinpoint the sensors  at fault in an industrial control system  application. 
In  \cite{ameli2022unsupervised}, the authors also suggest using the Shapley value 
for sensor anomaly localization, but test these ideas using a 
non-sensor server machine data set.  

{ \subsection{Shapley Value in 
Feature Localization}}
Other research attempts to localize which inputs to a machine learning algorithm most impact 
a particular output decision. We call this feature localization. These studies may or may not be related to sensors or anomaly detection.  
In \cite{XAIShapleyincomplexADML}, the Shapley value and simplifications of the Shapley value are used for feature localization in an anomaly detection application. 
In \cite{networkpacket}, a simplification of the Shapley value is utilized in network traffic data to identify which features are most important for some particular decisions. 
In \cite{characteristicfunction}, the Shapley value is used in tandem with a characteristic function for post-hoc feature localization. The algorithm is tested on different kinds of medical data, some of which may come from sensors. 

In \cite{PCA}, the Shapley value is used to localize reconstruction errors from a principal component analysis.   This is tested on various datasets ranging from cardio data, forest cover, radar returns, mammography and satellite imaging. 
The research in \cite{autoencoders} 
applied a simplification of the Shapley value for feature localization 
in autoencoder networks employed for 
anomaly detection.  Various datasets were used in the testing, including warranty claim datasets, credit card fraud detection, military network intrusion detection, and an artificial dataset.  
The research in \cite{modelindependentfeatureattribution} uses a Shapley value-based method for feature localization.   The approach is tested on artificial datasets and medical data.  
The research in \cite{explainingindividual} also employs a simplification of the  Shapley value 
for feature localization, while being tested on
simulated and real mortgage default data. 

{\subsection{Theoretical Analysis
of the Shapley Value for Machine Learning}}
The authors in \cite{owen2017shapleyvaluemeasuringimportanceanova} study feature localization by showing that it gives similar 
results as an analysis of variance method. 
In \cite{sundararajan2020many}, the authors compare different Shapley methods 
 theoretically and mathematically to  highlight their advantages for 
 different machine learning 
 models and applications. 
Most importantly, we have not seen any papers in the literature that study the issues enumerated in the last paragraph of the {Introduction}, thus justifying the novelty of this letter. }

\section{Analytical Results}

{ As per the previous discussion (Introduction, second new paragraph on page 2), to decide which of $\phi(x_i)$ and $v(x_i)$ (for any $1 \leq i \leq N$) is better at determining if $x_i$ is part of the anomaly, we compare the probability of error of two tests which each make a decision on if the anomaly includes  the $i$th sensor.  Each test will decide that the anomaly includes sensor $i$ if the function it compares to a threshold is larger than the threshold\footnote{{ Recall a larger value signifies a higher likelihood of an  anomaly.}}. Otherwise the test decides the anomaly does not include sensor $i$. The first test compares $ \phi(x_i)$ to an optimized threshold chosen to minimize the probability of error of this test.  The second test compares $ v(x_i) $ to an optimized threshold chosen to minimize the probability of error of this test. 
If the probability of error of the test using $ \phi(x_i)$ is smaller than the probability of error of the test using $ v(x_i) $, then $ \phi(x_i)$ is better at determining if $x_i$ is part of the anomaly. Otherwise  $ v(x_i) $ is better.} 


{ We make the following assumptions only for the following Theorem (Theorem III.1). }\begin{enumerate}
\item Assume the unattacked sensor data at a given time $x_1,x_2,...,x_N$ are statistically independent, each $x_i, i = 1,\ldots,N $ following the marginal probability density function (pdf) or probability mass function (pmf) $f_i(x_i)$.
\item { As discussed in the last paragraph of page 1, we define 
$v$  as the natural log of the reciprocal of the unattacked joint pdf/pmf of the sensor data corresponding to the indices in the input to $v$.  }
This holds regardless of if the data are statistically independent.  
\end{enumerate}

\begin{theorem}
\label{Th1}
Under assumptions 1 and 2, a test based on comparing the Shapley value $\phi(x_i)$ to an optimized threshold is exactly the same as a test based on comparing $v(x_i)$ to an optimized threshold. In both 
{ tests}, the threshold is optimized to minimize the probability of error for the given test. 
\end{theorem}
\begin{proof}
Recall 
\begin{equation}\label{shapeq}
    \phi (x_i)= \sum_{S \subseteq {\cal N} / (i)} \frac{ |S| ! (N- |S| -1)!}{N!} (v (S \cup (i)) - v(S)). 
\end{equation}
As per assumptions 1 and 2 the marginal pdf/pmf of 
$x_i$ is $f_i(x_i)$.  Given the assumed statistical independence, we find the 
joint pdf/pmf of 
$x_1,x_2,...,x_L$ is 
$ \prod_{j=1}^{L} f_j(x_j) $ for $L \leq N$. 
{ Thus for $i=N$ and $S=x_1,\ldots,x_{N-1}$,  
direct calculation 
using 
$ln(abc) = ln(a) + ln(b) + ln(c) $ and assumption 2 yields    
\begin{eqnarray}
    (v (S \cup (i)) - v(S)) &=&
 \ln{\left( \frac{1}{f_1(x_1),f_2(x_2), \cdots, f_N(x_N)}\right)} \nonumber \\ 
&-& \ln{\left( \frac{1}{f_1(x_1),\cdots, f_{N-1}(x_{N-1})} \right)} \nonumber \\ 
&=& \sum_{j=1}^{N} \ln{\left( \frac{1}{f_j(x_j)} \right)} \nonumber \\ &-& \sum_{j=1}^{N-1} \ln{\left( \frac{1}{f_j(x_j)} \right)} \nonumber \\ 
&=& \sum_{j=1}^{N} v(x_j) - \sum_{j=1}^{N-1} v(x_j) \nonumber \\ 
&=& v(x_N) 
\end{eqnarray}
Performing the same calculation for any valid $S$ will give exactly the same result. 
Thus from (\ref{shapeq}) 
\begin{eqnarray}\label{shapeq2}
     \phi (x_N) &=& v(x_N) 
    \sum_{S \subseteq {\cal N} / (N)} \frac{ |S| ! (N- |S| -1)!}{N!} \\ 
    &=& C v(x_N)
\end{eqnarray} 
where $C$ is a positive constant in $x_N$. This constant can depend on things other than $x_N$ but none of this changes the proof. 
Thus a test which decides for an anomaly if $ \phi (x_N) $ is greater than an optimum threshold $\tau$ is the same as a test comparing $ C v(x_N) $ to $\tau$. Note that this is the same as comparing $ v(x_N) $ to a threshold $\tau/C$. It follows that $\tau/C$ must be the optimum threshold for the optimum threshold test using $ v(x_N) $.  Thus the optimum threshold test using  $ \phi (x_N) $ must be exactly the same as the optimum threshold test using $ v(x_N) $. 
Similar evaluation for any valid $i$ ($ 1 \leq i \leq N$) shows 
$ \phi (x_i) =  C v(x_i)$ so 
these same conclusions hold for $ 1 \leq i \leq N$.
}\end{proof}
{ 
Thus under the assumptions of  
    Theorem~\ref{Th1}, the complexity analysis given at the end of the Introduction implies that using $ v(x_i)$ for anomaly localization performs as well as using $ \phi(x_i)$, with lower complexity. }

{
While we have restricted our attention to cases involving anomaly 
localization and a statistical formulation, we note that the results presented have implications for cases not involving anomaly 
localization or a statistical formulation as well. 
{ To demonstrate this, we next give a different   Theorem that holds for a certain class of classifiers that satisfy a certain condition relating $v(
\tilde{i}_1, \tilde{i}_2, ..., \tilde{i}_L
)$ and $ v(\tilde{i}_j
), j=1,\ldots,L $ for all $L\leq N$.  The results apply 
for any feature localization in a binary classification problem.  
Thus, the variables $\tilde{i}_1, \tilde{i}_2, ..., \tilde{i}_L
$ are feature indices (not necessarily sensor indices) and instead of 
making a decision about an anomaly we allow the decision to be any binary classification decision.  We make no assumptions about $v(\tilde{i}_1, \tilde{i}_2, ..., \tilde{i}_L
)
$, except those in the following Theorem.  This means we do not assume $x_{\tilde{i}_1}, x_{\tilde{i}_2}, ..., x_{\tilde{i}_L} 
$ are random.
\begin{theorem}
\label{Them-vfac}
Under the assumption that 
$v(\tilde{i}_1, \tilde{i}_2, ..., \tilde{i}_L
) = \sum_{j=1}^{L} v(\tilde{i}_j
) $ for all subsets of $L \leq N$ sensor indices $\tilde{i}_1, \tilde{i}_2, ..., \tilde{i}_L $, a test based on comparing the Shapley value $\phi(x_i)$ to an optimized threshold (same optimization as in Theorem III.1) is exactly the same as a test based on comparing $v(x_i)$ to an optimized threshold (same optimization as in Theorem III.1).
\end{theorem}}
{ 
\begin{proof}
The proof follows from that in Theorem~\ref{Th1} since the independence condition implies 
$v(\tilde{i}_1, \tilde{i}_2, ..., \tilde{i}_L
) = \sum_{j=1}^{L} v(\tilde{i}_j
) $ for all $L\leq N$ (steps in (3) illustrate this) and this is what leads to (3) so that (5) is true. 
\end{proof}
From Theorem~\ref{Them-vfac}, it follows, we should use $v(x_i)$ rather than the Shapley value to determine if $x_i$ is important in the binary classification when 
 $|v(\tilde{i}_1, \tilde{i}_2, ..., \tilde{i}_L
 ) - \sum_{j=1}^{L} v(\tilde{i}_j
) |$ is always sufficiently small for all subsets of data of size $L$ and for all possible $L$. This yields lower complexity with the same performance under the assumptions.}}

\section{Numerical Results}
 
{ 
Here, similar to the formulation in Theorem III.1, } we numerically compare the probability of error $P_e$ of a test that compares $\phi(x_i)$ to an optimized ($min P_e$)  threshold to that for 
a test that compares $v(x_i)$ to an optimized ($min P_e$) threshold.  
The two tests each make a decision on if the anomaly includes  the $i$th sensor. 
The better test will have a smaller $P_e$ and that implies 
that either $\phi(x_i)$ or $v(x_i)$ are better for localizing the anomaly. 
The optimum thresholds are found by searching over a fine grid.
In our numerical results, 
we use a Monte  Carlo simulation 
to approximate the 
probability of error, which is a standard approach in statistics. The approximation will be accurate for a large number of simulated data samples, called the number of Monte Carlo runs $M$, which we 
will employ.
Let the symbols $P_{e,\phi}$ and $P_{e,v}$ denote the probability of error for the test using $\phi(x_i)$ (the Shapley Value) and the probability of error for 
the test using $v(x_i)$, respectively. 

{ While we considered only statistically independent observation cases in Theorem III.1, we consider some 
statistically dependent observation cases also in the 
numerical results. 
In particular,  we consider cases with two sensors and we model an unattacked data sample  $x_1,x_2$ as following  the bivariate Gaussian pdf in 
\begin{eqnarray} 
    f(x_1,x_2) = \frac{1}{2 \pi \sigma_1 \sigma_2 \sqrt{1-\rho^2}} 
    exp\biggl{(}- \frac{1}{2(1-\rho^2)} \nonumber \\ \biggl{(}(\frac{x_1-\mu_1}{\sigma_1})^2+(\frac{x_2-\mu_2}{\sigma_2})^2-2\rho(\frac{x_1-\mu_1}{\sigma_1})(\frac{x_2-\mu_2}{\sigma_2})\biggr{)} \biggr{)}
    \label{bivarGauss}
\end{eqnarray}
where $(\mu_1,\mu_2)$ denotes the mean vector and 
$(\sigma^2_1,\sigma^2_2)$ is the variance vector.  If 
$\rho=0$ in 
(\ref{bivarGauss}), the two sensor samples $x_1, x_2$ are statistically independent and Gaussian distributed.  The symbols $\rho,  \sigma_1, \sigma_2$ for the unattacked data pdf that appear in 
(\ref{bivarGauss}) 
are used in the tables we present shortly. } 

We consider three different types of attacks, denoted by $A, B, C$. For type $A$, a constant value, called the  attack magnitude and denoted by $AM$, is added to the unattacked observation at sensor 1. For type $B$, a  Gaussian random variable is added to the unattacked observation at sensor 1.  The Gaussian random variable has mean $AM$ and a standard deviation $\sigma_a$. 
For type $C$, a uniform random variable is added to the unattacked observation at sensor 1. The uniform  random variable is 
the sum of a constant $AM$ and a zero-mean 
uniform random variable 
between $0$ and $UM$. 

In Table~\ref{tab:iid_gaussian}, we present results obtained from running Monte Carlo runs with { $M= 20,000,000$}, where we generate the unattacked sensor data as independent ($\rho=0$) and Gaussian distributed with $\mu_1=\mu_2=0$ and the values 
of $\sigma_1^2, \sigma_2^2$ shown in Table~\ref{tab:iid_gaussian}.  
{ We let half of the $M$ Monte Carlo runs have sensor 1 under attack while the other half do not.} 
The attack type and parameters ($AM$, $UM$ and $\sigma_a$)  are also shown in Table~\ref{tab:iid_gaussian}.  
In Table~\ref{tab:iid_gaussian}, we find that both $P_{e,\phi}$ and $P_{e,v}$ 
increase with an increase in 
$\sigma_1=\sigma_2$ (other things equal), which is as 
expected. The results in Table~\ref{tab:iid_gaussian} also follow the main results in Theorem~\ref{Th1} which says the two tests must be identical.  In Table~\ref{tab:iid_gaussian} we find $P_{e,\phi} = P_{e,v}$ for the same value of $\sigma_1=\sigma_2$, which would be the case if the two tests were identical. 

\begin{table}[t]
\centering
\caption{Probabilities of Error when unattacked samples $x_1, x_2$ at the two sensors are independent and Gaussian distributed.
Monte Carlo length = 20{,}000{,}000}
\label{tab:iid_gaussian}
\begin{tabular}{c c c c c c c}
\hline
$\sigma_1=\sigma_2$ & Attack & $\sigma_a$ & AM & UM 
& Num.\ $P_{e,v}$ & Num.\ $P_{e,\phi}$ \\
\hline
1.0 & A & na  & 10 & na  & $4.0000\times10^{-7}$ & $4.0000\times10^{-7}$ \\
1.5 & A & na  & 10 & na  & $6.0350\times10^{-4}$ & $6.0350\times10^{-4}$ \\
2.0 & A & na  & 10 & na  & $8.7218\times10^{-3}$ & $8.7218\times10^{-3}$ \\
1.0 & B & 0.1 & 10 & na  & $4.0000\times10^{-7}$ & $4.0000\times10^{-7}$ \\
1.5 & B & 0.1 & 10 & na  & $6.1085\times10^{-4}$ & $6.1085\times10^{-4}$ \\
2.0 & B & 0.1 & 10 & na  & $8.7563\times10^{-3}$ & $8.7563\times10^{-3}$ \\
1.0 & B & 1.0 & 10 & na  & $2.2300\times10^{-5}$ & $2.2300\times10^{-5}$ \\
1.5 & B & 1.0 & 10 & na  & $1.6676\times10^{-3}$ & $1.6676\times10^{-3}$ \\
2.0 & B & 1.0 & 10 & na  & $1.2493\times10^{-2}$ & $1.2493\times10^{-2}$ \\
1.0 & C & na  & 9.95 & 0.1 & $4.0000\times10^{-7}$ & $4.0000\times10^{-7}$ \\
1.5 & C & na  & 9.95 & 0.1 & $6.4379\times10^{-4}$ & $6.4379\times10^{-4}$ \\
2.0 & C & na  & 9.95 & 0.1 & $9.0456\times10^{-3}$ & $9.0456\times10^{-3}$ \\
\hline
\end{tabular}
\label{table:independent_results}
\end{table}
}

{

Next, we show that for one case with dependent observations, the anomaly localization test using the Shapley value performs very poorly. Here we 
consider bivariate Gaussian unattacked data $(x_1,x_2)$ (see (\ref{bivarGauss})) with $\rho=0.8,\sigma_1 = \sigma_2 = 2 $ and $\mu_1=\mu_2=0$. These results are for attack type A with $AM$ taking values in the set $\{1, 1.5, 2, \ldots, 20\}$. To choose the test thresholds, we consider the case where only sensor 2 is attacked half the time 
($M/2$ out of $M=20$ million) 
with the $AM$ values indicated and sensor 1 is 
never attacked.  As shown in 
Figure~\ref{fig1},  
the test using the Shapley value $\phi(x_1)$, which depends on both 
$x_1$ and $x_2$, incorrectly decides that $x_1$ is attacked with a probability $P_{FA}$ that grows towards one as $AM$ becomes larger. 
If we plot $P_{FA}$ for the test using $v(x_1)$ we find the 
curve is horixontal (flat) and if the threshold is picked reasonably (to give performance better than guessing)
that line is far below the horizontal line of PFA=0.5 so performance is much better than
that for the Shapley function.

The result in Figure~\ref{fig1} is not so surprising since the formula for $\phi(x_1)$ under the described attack does imply it should be an increasing function of $AM$ for sufficiently large $AM$.  
  \begin{figure}[t]
  \centering
  \includegraphics[width=0.9\columnwidth]{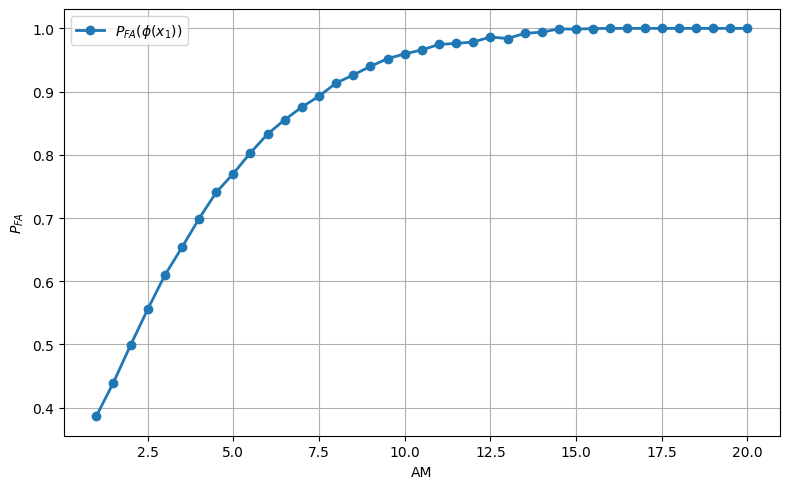}%
  \caption{Probability $P_{FA}$ that sensor~1 is deemed attacked by 
  Shapley test plotted versus attack type A attack magnitude $AM$ when only sensor~2 is actually attacked for cases with bivariate Gaussian unattacked sensor data $(x_1,x_2)$  with $\rho=0.8,\sigma_1^2 = \sigma_2^2 = 2 $ and $\mu_1=\mu_2=0$. }
  \label{fig1}
\end{figure} 
}

\section{Conclusion}
A  recent idea to employ the Shapley value for anomaly localization for sensor data systems  
is further studied. 
{
Using a reasonable statistical model for the classifiers required to compute the Shapley value, we found that 
using a single fixed  term in the Shapley value calculation $v(x_i)$, as opposed to the Shapley value, achieves 
a lower complexity anomaly localization test with an 
identical 
probability of error for all our  experiments { with independent observations}.  A proof demonstrates these results must be true for all independent observation cases. 
{
For some dependent observation cases with two sensors,  where only the second sensor data is anomalous/attacked, we show numerically that the Shapley value test can falsely decide an attack occurs at the first (unattacked) sensor with a  probability which approaches one for increasing attack magnitude. On the other hand, using the $ v(x_i) $ test always gives reasonably small probability of an attack occurring at the first (unattacked) sensor in these cases.}
Based on the existing literature we found, these are the first results of this type. }

{ Our results  have implications for some approximate Shapley value calculations using independent observations.  We now know that using the exact Shapley value is not as efficient as using the described (one term) classifier approach in such cases. 
Let's assume the accuracy of the approximate Shapley value calculation closely approximates the exact Shapley value. 
Thus, those approximate 
Shapley value calculations 
that remove less than all but one term from the Shapley value calculations  
will also be less efficient than using the described classifier approach since the classifier approach uses just one term. 
On the other hand, the accuracy of the approximate Shapley value calculations is not guaranteed which can also degrade their utility. } 
It would be nice to obtain some proofs for dependent observation cases.  
It would be nice to extend the study to other methods to identify the most important inputs to decision algorithms/AI   beyond Shapley.  { It would be nice to consider alternative statistical models for the classifiers required to compute the Shapley value and to further extend the study beyond anomaly localization. }

\appendices

\printbibliography
\end{document}